\documentclass[final,5p,times]{elsarticle}
\usepackage{amssymb,amsmath,amsfonts,amsthm,mathrsfs}
\usepackage{algorithm}
\usepackage{algpseudocodex}
\usepackage{booktabs}
\usepackage{xcolor}

\theoremstyle{plain}
\newtheorem{theorem}{Theorem}
\newtheorem{lemma}{Lemma}

\theoremstyle{remark}

\journal{Knowledge-Based Systems}

\makeatletter \def\ps@pprintTitle{ \let\@oddhead\@empty \let\@evenhead\@empty \def\@oddfoot{\today\hfil} \let\@evenfoot\@oddfoot}
\makeatother

\begin{document}

\begin{frontmatter}

\title{\LARGE Projected random forests and conformal prediction of circular data}

\author[1]{Paulo C. Marques F.\corref{cor1}}
\ead{PauloCMF1@insper.edu.br}

\author[1]{Rinaldo Artes}
\ead{RinaldoA@insper.edu.br}

\author[3]{Helton Graziadei}
\ead{helton@ufscar.br}

\affiliation[1]{
  organization = {Insper Institute of Education and Research},
  city = {São Paulo},
  postcode = {04546-042}, 
  state = {SP},
  country = {Brazil}
}

\affiliation[3]{
  organization = {Department of Statistics, Federal University of São Carlos},
  city = {São Carlos},
  postcode = {13565-905}, 
  state = {SP},
  country = {Brazil}
}

\cortext[cor1]{Corresponding author}

\begin{abstract}
We apply conformal prediction techniques to regression problems with circular responses, producing prediction sets with adaptive arc length and finite-sample coverage guarantees for any circular predictive model under the assumption of data exchangeability. Leveraging the high performance of existing predictive models designed for linear responses, we analyze a general projection procedure that converts any linear-response regression model into one suitable for circular responses. When random forests are used as base models in this projection procedure, we leverage the random forest out-of-bag mechanism to eliminate the need for a separate calibration sample in the construction of prediction sets. On synthetic and real datasets, the resulting projected random forest model produces more efficient out-of-bag conformal prediction sets, with shorter median arc length, than the split conformal prediction sets generated by two existing alternative models.
\end{abstract}



\end{frontmatter}

\section{Introduction}\label{sec:intro}

Circular variables are used to represent attributes in a variety of domains, including the occurrence time of events \cite{gill2010,leguia2021}, the orientation of phenomena in biological \cite{ali2023,landler2018}, meteorological \cite{lang2020}, ecological \cite{otieno2006,fitak2017,ranalli2020}, and geological \cite{lark2014} applications, alongside measures of psychological traits based on the interpersonal circumplex \cite{cremers2021}. What distinguishes these attributes is that their values are naturally represented as angles on a periodic scale, which makes standard linear methods inadequate and motivates the development of specialized statistical concepts and tools \cite{mardia2000,pewsey2021}.

In this study, we use conformal prediction techniques \cite{vovk2005} to quantify the confidence associated with predictions made by models designed for regression problems with circular responses. We also examine a general projection procedure that transforms any predictive model originally designed for linear responses into one suitable for circular settings. When applied to random forests \cite{breiman2001}, this projection procedure allows us to use the out-of-bag conformalization technique introduced in \cite{johansson2014}, which eliminates the need for a separate calibration sample.

We proceed as follows. In Section \ref{sec:circdata}, we review the statistical notions involved in modeling circular variables. In Section \ref{sec:cp}, we discuss the necessary conformal prediction concepts and then show in Section \ref{sec:score} how split conformal prediction \cite{papadopoulos2002} is implemented in our setting by proposing a suitable circular conformity score that produces model-free prediction sets with finite-sample coverage guarantees for exchangeable data. The general projection procedure used to build circular predictive models from existing methods for linear responses is discussed in Section \ref{sec:prf} and applied to random forests. The resulting projected random forest model is conformalized in a data-splitting-free manner by exploiting the availability of out-of-bag predictions for each training observation. We provide theoretical support for this out-of-bag conformal prediction procedure in Appendix B. In Section \ref{sec:experiments}, we compare the out-of-bag conformal prediction sets produced by the projected random forest model with the split conformal prediction sets generated by two alternatives: the parametric projected normal linear regression model \cite{presnell1998} and the semi-parametric circular forest \cite{lang2020}. On both synthetic and real datasets, the projected random forest model produces more efficient prediction sets, with shorter median arc length, over a batch of test predictions. We present our final remarks in Section \ref{sec:concl}, discussing future work and providing pointers to our open source \texttt{R} \cite{R} code and the data needed to reproduce all the presented examples.


\section{Circular data}\label{sec:circdata}

The distinguishing characteristic of a circular variable is that it can be represented as an angle or, equivalently, as a vector on the unit circle. This type of variable includes a wide range of measurements with a natural periodic structure, from inherently angular quantities, such as the wind direction at a particular geographical location, to variables that record the time at which events occur. The cyclical nature of circular variables requires us to rethink even the most basic descriptive concepts. For instance, the average time of occurrence for events observed at 23:00 and 1:00 is not 12:00, as a linear average would suggest, but rather midnight. The same issue arises with the example of wind directions: observations at $350^\circ$, $10^\circ$, $30^\circ$, and $50^\circ$ have average direction $20^\circ$, a direction slightly east of north, rather than the misleading linear average of $110^\circ$.

Let $\mathbf{atan2}:\mathbb{R}^2\setminus\{(0,0)\}\to[0,2\pi)$ be a two-argument arctangent, returning the polar angle in $[0,2\pi)$ associated with the vector $(c,s)$, defined by:
\begin{equation*}
    \mathbf{atan2}(c,s) =
    \begin{cases}
      \arctan(s/c) &\text{if} \;\; c > 0 \; \text{and} \; s \ge 0; \\
      \arctan(s/c) + \pi &\text{if} \;\; c < 0; \\
      \arctan(s/c) + 2\pi &\text{if} \;\; c > 0 \;\text{and} \; s < 0; \\
      \pi/2 &\text{if} \;\; c = 0 \;\text{and} \; s > 0; \\
      3\pi/2 &\text{if} \;\; c = 0 \;\text{and} \; s < 0.
    \end{cases}
\end{equation*}

Suppose that we have observed at a specific location a sample of wind directions represented as angles $y_1,\dots,y_n \in [0,2\pi)$. A natural notion of mean wind direction associated with this sample is the circular sample mean, defined by summing the corresponding unit vectors $(\cos(y_i),\sin(y_i))$ and taking the polar angle of the resulting vector:
\[
  \bar{y}_\text{circular} = \mathbf{atan2}\left(\sum_{i=1}^n \cos(y_i),\sum_{i=1}^n\sin(y_i)\right),
\]
with the convention that this circular sample mean is not defined when the resulting vector obtained from the vector sum is the zero vector. Moreover, the length of the vector sum of these unit vectors provides a measure of sample concentration. If all observations are equal, the sample has maximal homogeneity and the resultant length is $n$. Conversely, if the angles $y_i$ are evenly spread around the circle, the sample has maximal dispersion and the resultant length is close to zero, reaching zero in the case of a perfectly symmetric configuration of the unit vectors.

There are several probability distributions suitable for modeling a circular random variable \cite{mardia2000}. The von Mises distribution plays a central role in the field and is often regarded as a circular analogue of the normal distribution \cite{gordon1977}. A circular random variable $Y\in[0,2\pi)$ following a von Mises distribution with mean direction $\theta$ and concentration parameter $\kappa>0$ has probability density function
\begin{equation*}
  f_Y(y)=\frac{\exp(\kappa \cos(y-\theta))}{2 \pi I_0(\kappa)},
\end{equation*}
for $0\leq y<2\pi$, in which $I_0$ is the modified Bessel function of the first kind of order $0$. As $\kappa$ increases, the von Mises distribution becomes increasingly concentrated around $\theta$. When $\kappa$ approaches zero, it converges to the circular uniform distribution, for which no direction is favored. It can be shown \cite{mardia2000} that, for a random sample from a von Mises distribution, the maximum likelihood estimator of the mean direction $\theta$ is the circular sample mean defined above. In Section \ref{sec:experiments}, we use the von Mises distribution in the definition of the data-generating process for our synthetic example.

\section{Conformal prediction}\label{sec:cp}

Conformal prediction \cite{vovk2005} is a framework for turning the output of a predictive model into a prediction set with a finite-sample coverage guarantee. Its main advantage is that this guarantee does not require the predictive model to be correctly specified. Instead, conformal prediction relies on a distributional symmetry assumption on the data to calibrate prediction sets around the predictions produced by any fitted model, making it particularly well suited to use with modern machine learning methods, which can achieve strong empirical performance but usually do not provide a direct and reliable quantification of predictive uncertainty. Here, we focus on split conformal prediction \cite{papadopoulos2002}, also known as the inductive form of conformal prediction.

A sequence of random objects $\{O_i\}_{i\geq 1}$ is \textit{exchangeable} if, for every $n\geq 1$ and every permutation $\pi$ of $\{1,\dots,n\}$, the random tuples $(O_1,\dots,O_n)$ and $(O_{\pi(1)},\dots,O_{\pi(n)})$ have the same distribution.

For our purposes, consider a regression setting in which the available data is randomly split into training and calibration samples, defining the data sequence
\begin{equation*}\label{eq:seq}
  \underbrace{(X'_1,Y'_1),\dots,(X'_m,Y'_m)}_\text{training},\underbrace{(X_1,Y_1),\dots,(X_n,Y_n)}_\text{calibration},\underbrace{(X_{n+1},Y_{n+1}),\dots}_\text{future}
\end{equation*}
with feature vectors $X'_i,X_i\in\mathbb{R}^d$ and response variables $Y'_i,Y_i\in\mathbb{R}$. We model the random pairs in this data sequence as being exchangeable.

Using the information in the training sample, we construct a \textit{conformity function} $\rho:\mathbb{R}^d\times\mathbb{R}\to\mathbb{R}$ whose role is to compare, within the exchangeable data sequence, predictions made from feature vectors with the corresponding response values.

For example, we may first learn a regression function $\hat\mu:\mathbb{R}^d\to\mathbb{R}$ from a realization of the training sample $\{(X'_i,Y'_i)\}_{i=1}^m$, and then define the conformity function as the absolute residual:  $\rho(x,y)=|y-\hat{\mu}(x)|$.

The distinctive theoretical feature of split conformal prediction follows from the following simple result, proved in Appendix A.

\begin{lemma}\label{lmm:order}
Let $U_1,U_2,\dots,U_n,U_{n+1}$ be exchangeable random variables, and let $V_{(1)}\leq V_{(2)}\leq\dots\leq V_{(n)}$ denote the order statistics of the subset $\{U_1,U_2,\dots,U_n\}$. For each ${k=1,\dots,n}$, it follows that $P(U_{n+1}\leq V_{(k)})\geq k/(n+1)$.
\end{lemma}

Let $\lceil t\rceil=\min\{k\in\mathbb{Z}:t\le k\}$ denote the ceiling of $t\in\mathbb{R}$. For a given conformity function $\rho$, define the conformity scores $R_i=\rho(X_i,Y_i)$, for $i\ge 1$. The assumption of an exchangeable data sequence implies that the full set of conformity scores $\{R_1,R_2,\dots,R_n,R_{n+1}\}$ is exchangeable. Choose a nominal miscoverage level $0<\alpha<1$ such that $\lceil(1-\alpha)(n+1)\rceil \leq n$, and denote by $R_{(1)}\leq R_{(2)}\le\dots\le R_{(n)}$ the order statistics of the calibration scores $\{R_1,R_2,\dots,R_n\}$. Introducing the notation $\hat{r} = R_{\left(\lceil (1-\alpha)(n+1)\rceil\right)}$, a direct application of Lemma \ref{lmm:order} with $k=\lceil(1-\alpha)(n+1)\rceil$ entails that
\begin{equation}\label{eq:mvp}
  \Pr(Y_{n+1}\in C^{(\alpha)}_n(X_{n+1})) \ge 1 - \alpha,
\end{equation}
in which the prediction set $C^{(\alpha)}_n(x)=\{y\in\mathbb{R}:\rho(x,y)\le \hat r\}$, with $x\in\mathbb{R}^d$. 

Property (\ref{eq:mvp}) is referred to as the marginal validity of conformal prediction sets. Intuitively, the conformalization process works like this: the calibration sample scores assess the out-of-sample predictive capacity of the model $\hat{\mu}$ through the conformity function $\rho$, and the distributional symmetry implied by the data exchangeability assumption transfers this assessment to the future observables, for which we become able to construct prediction sets with coverage guarantees without additional assumptions. In the next section, we introduce a conformity function suitable for regression problems with circular response variables.


\section{Conformity score for circular predictions}\label{sec:score}

Throughout this section, we follow the notations introduced in Section \ref{sec:cp}. We now consider the case in which the response variables are circular and expressed in radians, so that $Y'_i,Y_i\in[0,2\pi)$. The absolute difference between two given angles $\theta,\phi\in[0,2\pi)$ is measured through the angular distance
\begin{align*}
  d(\theta,\phi) &= \min\{|\theta-\phi|,2\pi-|\theta-\phi|\} \\
  &= \pi-|\pi-|\theta-\phi||\in[0,\pi],
\end{align*}
which gives the length of the shortest arc between $\theta$ and $\phi$ in the unit circle.

From the information in the training sample $\{(X'_i,Y'_i)\}_{i=1}^m$, we build a first predictive model $\hat{\mu}:\mathbb{R}^d\to[0,2\pi)$. Defining the circular training residuals $\Delta'_i = \pi-|\pi-|Y'_i-\hat{\mu}(X'_i)||$, a second positive variability model $\hat{\sigma}:\mathbb{R}^d\to(0,\pi]$ is constructed from the information in $\{(X'_i,\Delta'_i)\}_{i=1}^m$. We then define the conformity function as the scaled absolute circular residual:
\[ \rho(x,y)=\frac{\pi-|\pi-|y-\hat{\mu}(x)||}{\hat{\sigma}(x)} .\]
Consequently, the conformity scores for calibration and future sample units are given by
\[
  R_i = \rho(X_i,Y_i) = \frac{\pi-|\pi-|Y_i-\hat{\mu}(X_i)||}{\hat{\sigma}(X_i)}, \qquad \text{for $i\geq 1$.} 
\]
The presence of the variability model $\hat{\sigma}$ in the conformity score denominator yields prediction sets with variable arc length for a batch of future observables, making the conformalization procedure more adaptable to the values of the features \cite{papadopoulos2002} and helping with issues of conditional coverage as discussed at the end of Section \ref{sec:experiments}.

It follows from the marginal validity property (\ref{eq:mvp}) that
\[
  P\left(Y_{n+1} \in C^{(\alpha)}_n(X_{n+1})\right) \ge 1 - \alpha,
\]
in which, for $x\in\mathbb{R}^d$, the conformal prediction set is given by
\[
  C^{(\alpha)}_n(x) = 
  \begin{cases}
    \left[\hat{\mu}(x)-\hat{r}\cdot\hat{\sigma}(x),\, \hat{\mu}(x)+\hat{r}\cdot\hat{\sigma}(x)\right], & \text{if} \;\;\; \hat{r}\cdot\hat{\sigma}(x) < \pi; \\
    [0, 2\pi), & \;\;\; \text{otherwise}.
  \end{cases}
\]

\section{Projected random forests}\label{sec:prf}

The availability of high-performance open source implementations of contemporary machine learning methods for regression problems with linear response motivates the development of a general projection procedure that allows such methods to be used as building blocks in problems of circular regression. In this projection procedure, any predictive method designed for real-valued responses is applied separately to the cosine and sine projections of the circular response variable values, and predictions are made by projecting the resulting two-dimensional output back to $[0,2\pi)$.

Formally, we construct two predictive models with linear responses,
\[
  \hat{\mu}_c:\mathbb{R}^d\to\mathbb{R} \qquad\text{and}\qquad \hat{\mu}_s:\mathbb{R}^d\to\mathbb{R},
\]
using, respectively, the realizations of the projected training samples
\[
  \{(X'_i,\cos(Y'_i))\}_{i=1}^m \qquad\text{and}\qquad \{(X'_i,\sin(Y'_i))\}_{i=1}^m.
\]
Given a feature vector $x\in\mathbb{R}^d$, these two models produce the two-dimensional prediction $(\hat{\mu}_c(x),\hat{\mu}_s(x))$. The projected predictive model $\hat{\mu}:\mathbb{R}^d\to[0,2\pi)$ is then defined by taking the polar angle of this vector: $\hat{\mu}(x) = \mathbf{atan2}\left(\hat{\mu}_c(x),\hat{\mu}_s(x)\right)$.


A natural question is whether this projection procedure produces an actual predictive gain when compared to a regression mo\-del designed for linear responses forced to operate in a circular response setting. This comparison is especially transparent for random forests \cite{breiman2001}. Individual regression trees make predictions by averaging the response values associated with the training observations falling in each terminal node. A regression random forest, in turn, averages the predictions of the individual trees built during the bootstrap process. Consequently, by construction, a regression random forest cannot predict values outside the range of the response values observed in the training sample. Thus, when trained directly on angles encoded in $[0,2\pi)$, a regression random forest will also produce predictions in $[0,2\pi)$.

For the wind direction dataset discussed in Section \ref{ssec:wind}, we compared the test error of a standard regression random forest with that of projected random forests, that is, the predictive model built from two standard random forests using the projection procedure discussed above. We split the available data into training and test samples, with 75\% of the observations used for training and 25\% used for testing. Using the average angular distance between predicted and observed angles in the test sample as a measure of predictive performance, we observe a substantial performance gain in favor of the projected random forests model. In particular, the standard random forest yielded an average angular distance of 0.759, whereas the projected random forest reduced this value to 0.551. This corresponds to a relative reduction of approximately 27.4\% in the average angular prediction error. This highlights the predictive benefit of accounting for the circular nature of the response through the projection procedure.


The bootstrap process involved in the construction of a random forest allows each training observation to be compared with predictions made by trees that did not use its information during training. When a bootstrap sample of size $n$ is drawn from a training sample of size $n$, any particular observation is not selected with probability $(1-1/n)^n$, which is approximately $e^{-1}\approx 36.8\%$ for large $n$. Thus, in a random forest, each training observation is typically absent from the bootstrap samples used to construct approximately $36.8\%$ of the trees (a fraction humorously referred to by statisticians and machine learners as ``approximately one third''). For one such observation, these trees form a random subforest, whose predictions can be averaged to produce a prediction based only on models trained without access to the observation in question. We say that the observation stayed out-of-bag for the trees in this subforest, and the resulting prediction is its out-of-bag prediction.

In conformal prediction, this out-of-bag mechanism can be leveraged to compute calibration scores from the training sample itself, replacing the predictions made for the separate calibration sample by out-of-bag predictions. This out-of-bag conformal prediction procedure was first proposed in \cite{johansson2014} and has recently been successfully applied in actuarial problems \cite{graziadei2025}. In Appendix B, we give theoretical results for the out-of-bag conformal prediction procedure in a general setting, beyond the circular problems discussed in the paper.

Algorithm \ref{algo:prf} synthesizes the two main ideas discussed in this section, formalizing an out-of-bag conformalization procedure based on the projected predictive model built from two standard random forests.

\begin{algorithm}[t!]
\caption{Out-of-bag conformal prediction using projected random forests}\label{algo:prf}
\begin{algorithmic}[1]
  \Require Dataset $\{(x_i,y_i)\}_{i=1}^n$, with $x_i\in\mathbb{R}^d$ and $y_i\in[0,2\pi)$, number $B$ of trees used to train each random forest, random seed $\tau\in\mathbb{N}$, future vector of predictors $x_{n+1}\in\mathbb{R}^d$, and nominal miscoverage level $0<\alpha<1$.
  \Ensure Prediction set.
  \Statex
  \State Set the random seed to $\tau$
  \State Train random forest $\{\hat{\mu}^{(j)}_c\}_{j=1}^B$ from $\{(x_i,\cos(y_i))\}_{i=1}^n$
  \State Reset the random seed to $\tau$
  \State Train random forest $\{\hat{\mu}^{(j)}_s\}_{j=1}^B$ from $\{(x_i,\sin(y_i))\}_{i=1}^n$
  \For{$i=1 \textrm{ to } n$}
    \State $\mathcal{O}_i \gets \{ j : \text{$i$-th sample unit $\notin$ $j$-th bootstrap sample}\}$
    \State {\footnotesize $\delta_i \gets \pi - \left|\pi - \left|y_i-\mathbf{atan2}\!\left(\frac{1}{|\mathcal{O}_i|}\sum_{j\in\mathcal{O}_i}\hat{\mu}^{(j)}_c(x_i),\,\frac{1}{|\mathcal{O}_i|}\sum_{j\in\mathcal{O}_i}\hat{\mu}^{(j)}_s(x_i)\right)\right|\right|$ }
  \EndFor
  \State Reset the random seed to $\tau$
  \State Train random forest $\{\hat{\sigma}^{(j)}_c\}_{j=1}^B$ from $\{(x_i,\cos(\delta_i))\}_{i=1}^n$
  \State Reset the random seed to $\tau$
  \State Train random forest $\{\hat{\sigma}^{(j)}_s\}_{j=1}^B$ from $\{(x_i,\sin(\delta_i))\}_{i=1}^n$
  \For{$i=1 \textrm{ to } n$}
    \State $r_i \gets \delta_i \;\Big/\; \mathbf{atan2}\!\left(\frac{1}{|\mathcal{O}_i|}\sum_{j\in\mathcal{O}_i}\hat{\sigma}^{(j)}_c(x_i),\,\frac{1}{|\mathcal{O}_i|}\sum_{j\in\mathcal{O}_i}\hat{\sigma}^{(j)}_s(x_i)\right)$
  \EndFor
  \State $\hat{y}_{n+1} \gets \mathbf{atan2}\!\left(\frac{1}{B}\sum_{j=1}^B\hat{\mu}^{(j)}_c(x_{n+1}),\,\frac{1}{B}\sum_{j=1}^B\hat{\mu}^{(j)}_s(x_{n+1})\right)$
  \State {\footnotesize $\epsilon \gets r_{(\lceil(1-\alpha)(n+1)\rceil)} \times \mathbf{atan2}\!\left(\frac{1}{B}\sum_{j=1}^B\hat{\sigma}^{(j)}_c(x_{n+1}),\,\frac{1}{B}\sum_{j=1}^B\hat{\sigma}^{(j)}_s(x_{n+1})\right)$ }
  \If{$\epsilon<\pi$}
    \State\Return $[\hat{y}_{n+1} - \epsilon, \hat{y}_{n+1} + \epsilon]$
  \Else
    \State\Return $[0,2\pi)$
  \EndIf
\end{algorithmic}
\end{algorithm}

\section{Experiments with synthetic and real data}\label{sec:experiments}

In this section, we compare the split conformal prediction sets produced by the projected normal linear model and the circular forest (described in Sections \ref{ssec:pnlm} and \ref{ssec:cf}, respectively) to the out-of-bag conformal prediction sets generated by the projected random forests model described in Algorithm \ref{algo:prf}.

\subsection{Projected normal linear model}\label{ssec:pnlm}

The projected normal linear model \cite{presnell1998} introduces a set of real parameters and defines the response variable as the angle generated by a unit vector. This unit vector is constructed by normalization of a vector following a bivariate normal distribution, with mean vector components defined by linear combinations involving the model parameters and the explanatory variables, and with an identity covariance matrix. The model parameters are estimated by maximum likelihood. See \cite{presnell1998} for the technical details.

\subsection{Circular forest}\label{ssec:cf}

The circular forest model \cite{lang2020} is a bagging \cite{breiman1996} of circular trees. Each circular tree is built by a recursive partitioning process involving the von Mises distribution discussed in Section \ref{sec:circdata}. The process begins by fitting a von Mises distribution to the responses of the full training sample and generating a scoring matrix that measures how well the von Mises distribution fits each training sample unit. This scoring matrix is then tested for dependencies with each explanatory variable using permutation tests, thereby selecting one explanatory variable for the split. The best split point is found by maximizing a score discrepancy between the two resulting subgroups. This procedure is recursively applied to create further splits until a stopping criterion, such as minimal node size or lack of significant dependency, is met.

\subsection{Synthetic data}\label{ssec:synth}

The synthetic data generating process is the following. We simulate ten independent predictors $x_1,\dots,x_{10}$ from a $U[-1,1]$ distribution. The response variable $y\in[0,2\pi)$ is drawn from a von Mises distribution with mean direction
\[
  2\arctan\!\left(x_1 - 2x_2 + x_1x_2 - 2x_3^2\right)+\pi
\]
and concentration parameter $\kappa>0$. The predictors $x_4,\dots,x_{10}$ are introduced to create a degree of sparsity in the simulated dataset.

In this experiment, the training, calibration, and test samples each contain 10,000 observations. Figure \ref{fig:synthpolar} depicts the circular histogram for the response variable in the training sample, generated with concentration parameter $\kappa=5$. When using Algorithm \ref{algo:prf}, we enlarge the training sample by adding to it the calibration sample, since the algorithm does not need a separate calibration sample to produce the prediction sets. The corresponding prediction intervals for fifty test sample units, produced by the three different methods, using a miscoverage level $\alpha=0.1$, are shown in Figure \ref{fig:synthint}. In this figure we unfolded the prediction intervals to facilitate the arc length comparisons between the different models. For the whole test sample, Table \ref{tab:synth} gives the median and interquartile range of the arc lengths, and the empirical coverage for a batch of conformal prediction sets produced by the three different methods, for data generated with concentration parameters $\kappa\in\{1,2,5,10\}$, using a nominal miscoverage level $\alpha=0.1$. In all cases, Algorithm \ref{algo:prf} produces more efficient prediction intervals, with smaller median arc length.

\begin{figure}[t]
\centering
\includegraphics[width=5cm]{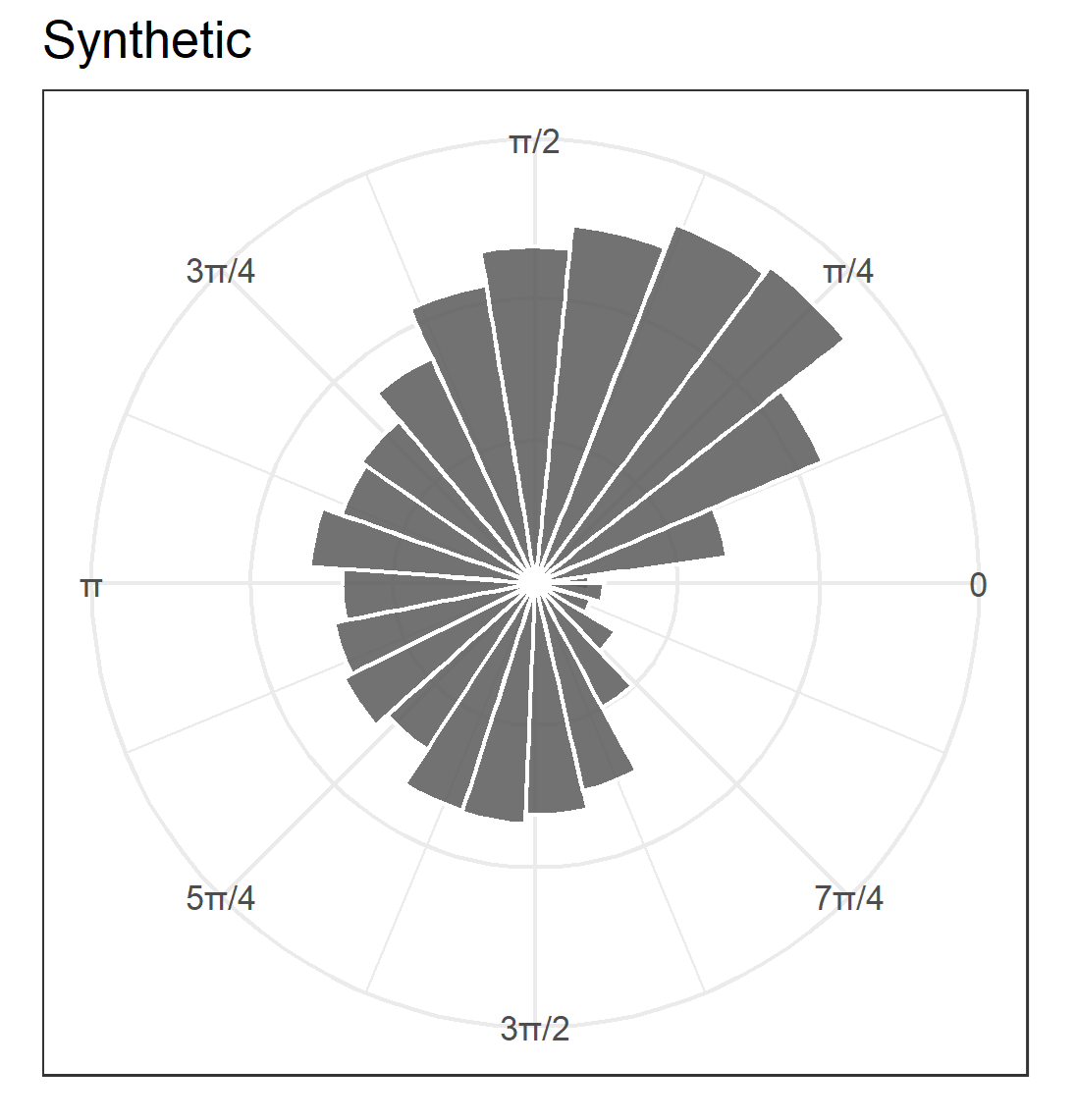} 
\caption{Circular histogram of the response variable in the synthetic dataset training sample, with concentration parameter $\kappa=5$.}
\label{fig:synthpolar}
\end{figure}

\begin{figure*}[ht]
\centering
\includegraphics[width=18cm]{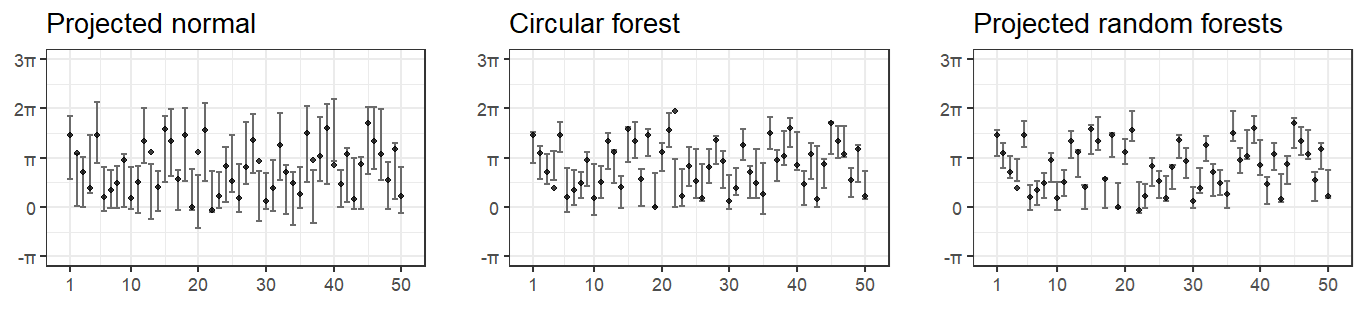}
\caption{Conformal prediction intervals for fifty test sample units in the synthetic dataset, with $\kappa=5$, produced by the three different methods, using a miscoverage level $\alpha=0.1$. The black dots are the observed circular responses.}
\label{fig:synthint}
\end{figure*}

\begin{table*}[ht]
\caption{Test sample median and interquartile range of the arc lengths, and empirical coverage for a batch of conformal prediction sets produced by the three different methods for synthetic data generated with different concentration parameters $\kappa$, using nominal miscoverage level $\alpha=0.1$.}
\medskip
\centering
\footnotesize
{\renewcommand{\arraystretch}{1.25}
\begin{tabular}{cccccccccc}
\cline{2-10}
                              & \multicolumn{3}{c|}{Projected normal}           & \multicolumn{3}{c|}{Circular forest}            & \multicolumn{3}{c}{Projected random forests} \\ \hline
\multicolumn{1}{c|}{$\kappa$} & Median & IQR    & \multicolumn{1}{c|}{Coverage} & Median & IQR    & \multicolumn{1}{c|}{Coverage} & Median       & IQR          & Coverage       \\ \hline
1                             & 4.90   & (0.74) & 90.5\%                        & 4.84   & (1.06) & 90.6\%                        & 4.51         & (0.71)       & 90.8\%         \\
2                             & 4.05   & (0.87) & 90.2\%                        & 3.52   & (0.85) & 90.4\%                        & 2.96         & (0.41)       & 90.1\%         \\
5                             & 3.40   & (1.20) & 90.4\%                        & 2.09   & (0.48) & 89.9\%                        & 1.70         & (0.25)       & 90.0\%         \\
10                            & 3.21   & (1.35) & 89.6\%                        & 1.70   & (0.62) & 90.1\%                        & 1.24         & (0.21)       & 89.9\%         \\ \hline
\end{tabular}}
\label{tab:synth}
\end{table*}

\subsection{Wind direction data}\label{ssec:wind}

We compiled hourly wind direction data from a meteorological station located in the Central-West region of Brazil, taking a sample through the period from January 1st, 2012, to July 30th, 2023. The raw data is publicly available at:
\begin{center}
  \texttt{\small https://tempo.inmet.gov.br/TabelaEstacoes/A001}
\end{center}
Table \ref{tab:windvars} gives a description of the available variables.

We have training, calibration, and test sample units of sizes 10,000, 5,000, and 5,000, respectively. For a discussion of calibration sample size selection, see \cite{marquesf2025a}. Figure \ref{fig:windpolar} shows the circular histogram of the response variable in the training sample. Again, Algorithm \ref{algo:prf} allows us to enlarge the training sample by adding to it the calibration sample. The prediction intervals for fifty test sample units, produced by the three different methods, using a miscoverage level $\alpha=0.1$, are shown in Figure \ref{fig:windint}. In this figure we unfolded the prediction intervals to ease the arc length comparisons between the different models. Figure \ref{fig:clockswind} displays the prediction intervals in a more traditional way. For the whole test sample, Table \ref{tab:wind} gives the median and interquartile range of the arc lengths, and the empirical coverage of the conformal prediction sets produced by the three different methods, using a nominal miscoverage level $\alpha=0.1$. Again, Algorithm \ref{algo:prf} outperforms the alternatives, producing prediction intervals with smaller median arc length.

\begin{table}[ht]
\small
\caption{Variables in the wind direction dataset.}
\medskip
\centering
{\renewcommand{\arraystretch}{1.25}
\begin{tabular}{lc}
\hline
Variable & Unit \\
\hline
Wind direction & rad \\
Cosine of wind direction in the previous hour & dimensionless \\ 
Sine of wind direction in the previous hour & dimensionless \\  
Total precipitation in the previous hour & mm \\ 
Atmospheric pressure in the previous hour & mB \\
Air temperature (dry bulb) in the previous hour	& $^\circ\mathrm{C}$ \\
Dew point temperature in the previous hour & $^\circ\mathrm{C}$ \\
Relative humidity in the previous hour & \% \\
Wind gust in the previous hour & m/s \\
Wind speed in the previous hour	& m/s \\
\hline
\end{tabular}}
\label{tab:windvars}
\end{table}

\begin{figure}[ht]
\centering
\includegraphics[width=5cm]{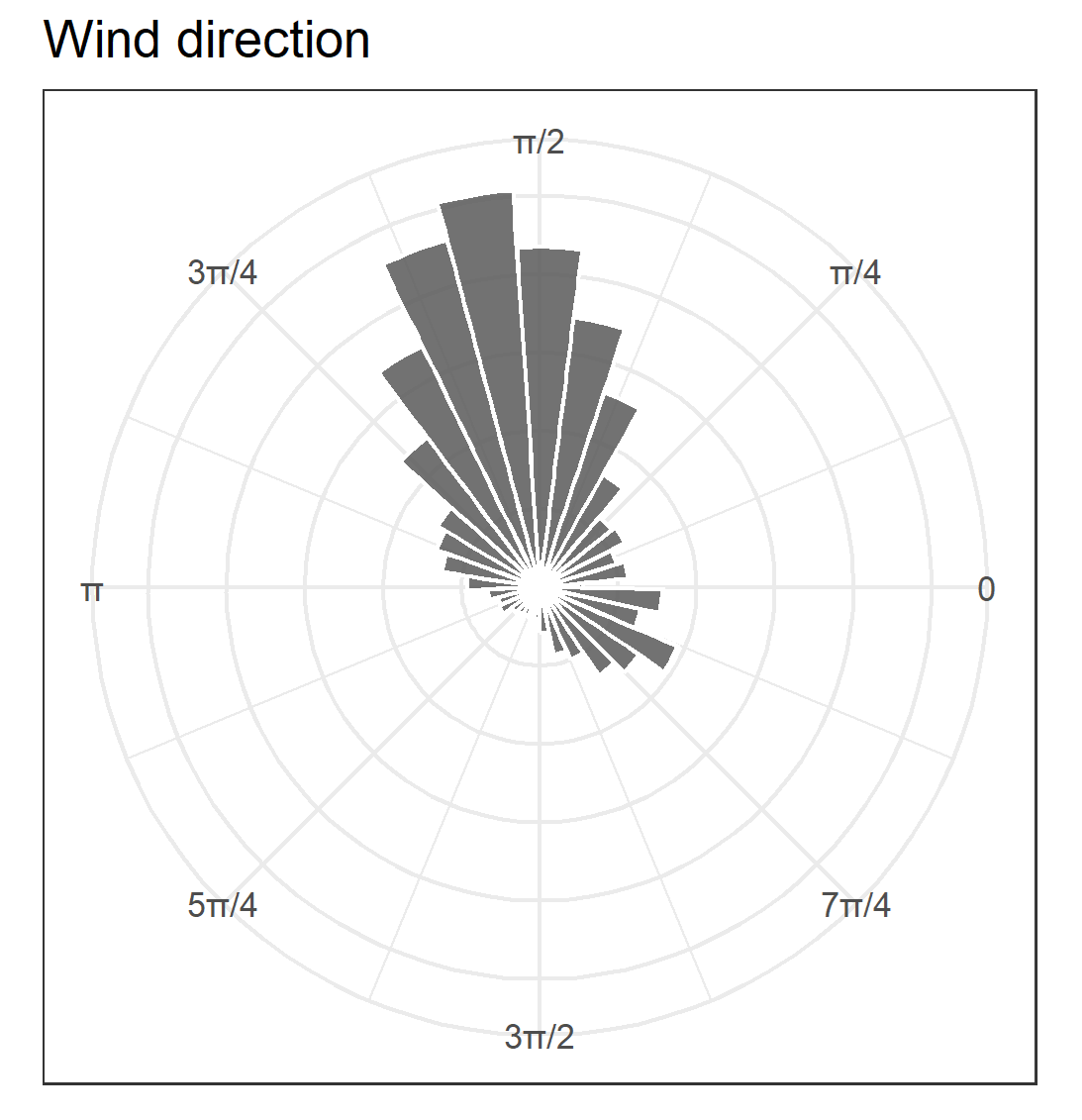}
\caption{Circular histogram of the response variable in the training sample of the wind direction dataset.}
\label{fig:windpolar}
\end{figure}

\begin{figure*}[ht]
\centering
\includegraphics[width=18cm]{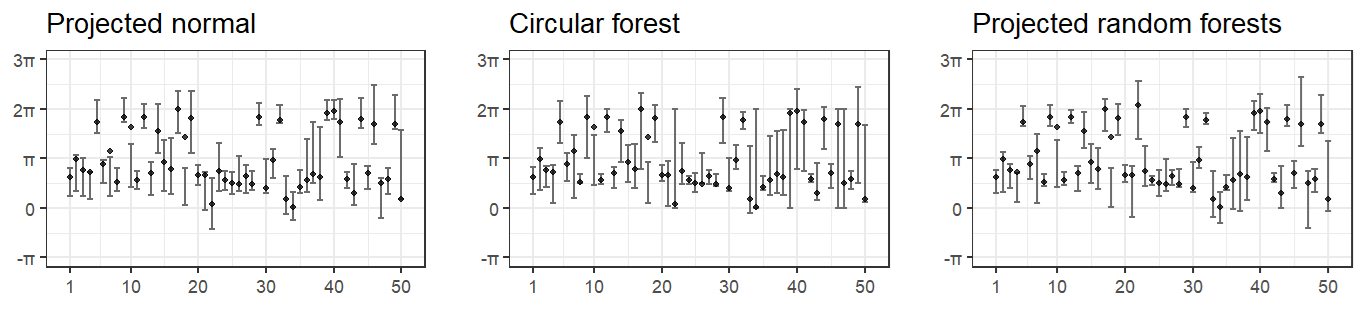}
\caption{Prediction intervals for fifty test sample units in the wind direction dataset, produced by the three different methods, using a miscoverage level $\alpha=0.1$. The black dots are the observed wind directions.}
\label{fig:windint}
\end{figure*}

\begin{table}[ht]
\caption{Median and interquartile range of the arc lengths, and empirical coverage for a batch of conformal prediction sets produced by the three different methods for the wind direction dataset, using nominal miscoverage level $\alpha=0.1$.}
\medskip
\centering
\small
{\renewcommand{\arraystretch}{1.25}
\begin{tabular}{lccc}
\cline{2-4}
                         & Median & IQR  & Coverage \\ \hline
Projected normal         & 2.04   & 1.39 & 89.2\%   \\
Circular forest          & 2.23   & 3.07 & 90.2\%   \\
Projected random forests & 1.90   & 1.87 & 89.5\%   \\ \hline
\end{tabular}}
\label{tab:wind}
\end{table}

\begin{figure*}[ht]
\centering
\includegraphics[width=14cm]{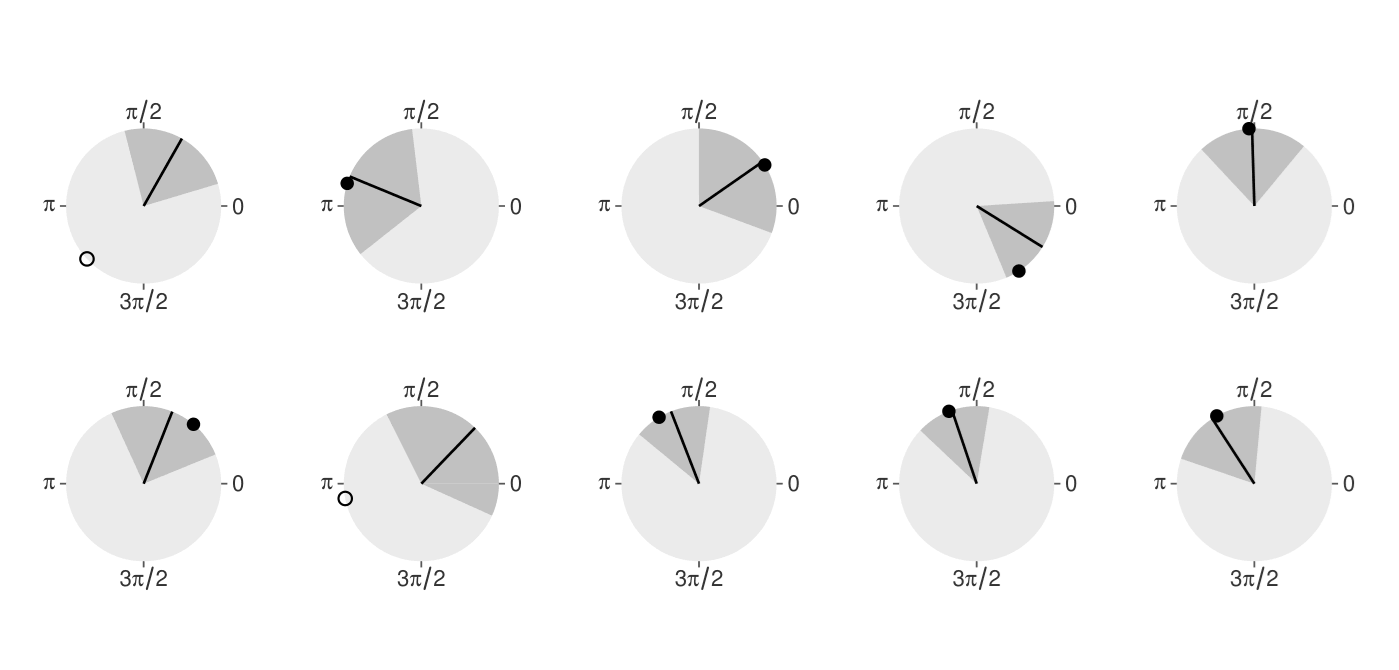}
\caption{Circular prediction intervals for wind direction obtained with the projected random forest model. Each polar diagram represents one test observation. The dark-gray sector indicates the conformal prediction interval, the radial black line represents the point prediction, and the point on the circumference corresponds to the observed wind direction. Filled points indicate observations covered by the prediction interval, while hollow points indicate observations not covered by it.}
\label{fig:clockswind}
\end{figure*}

\begin{figure}[ht]
\centering
\includegraphics[width=8cm]{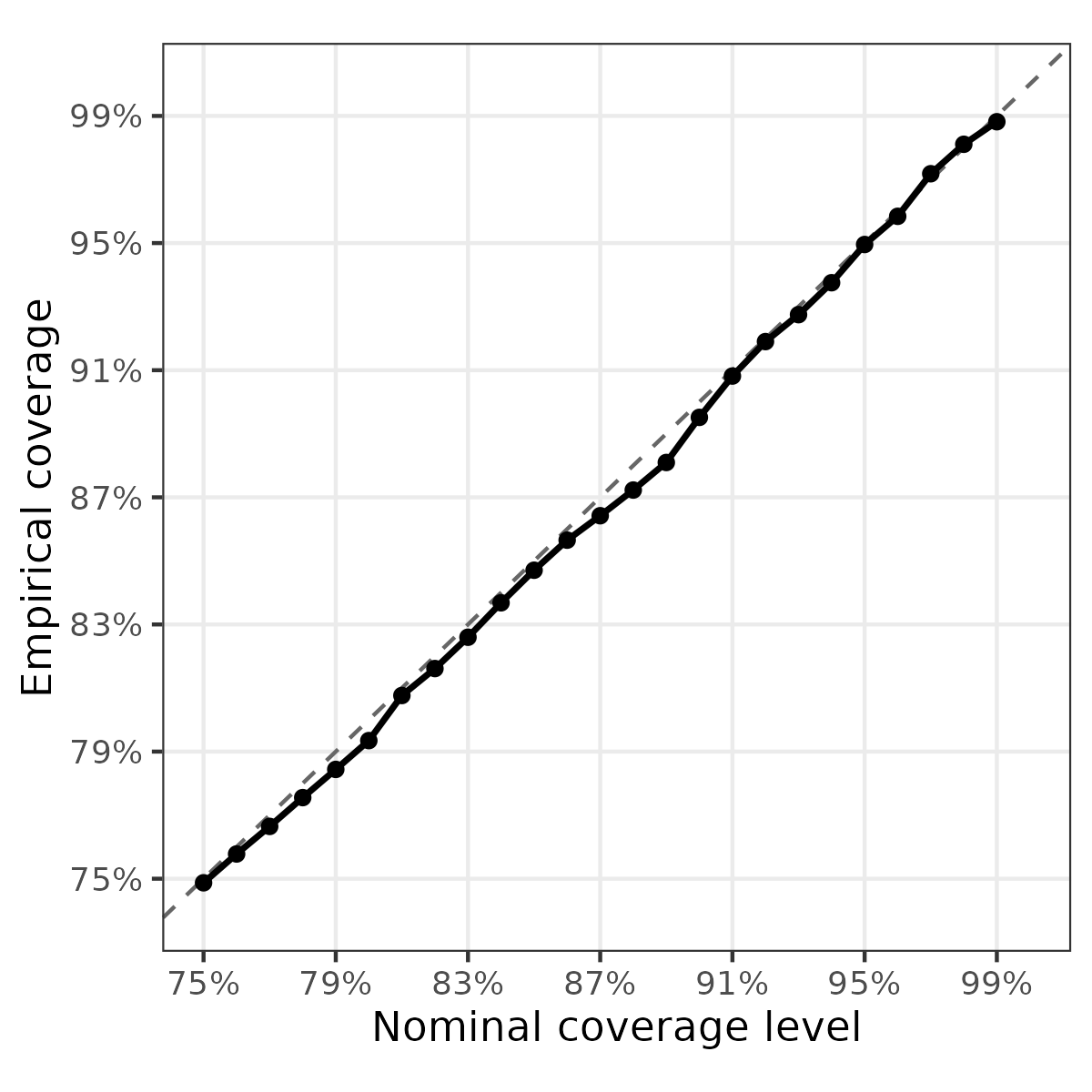}
\caption{Calibration plot for the projected random forest model for the wind direction data. Empirical coverage is shown as a function of the nominal coverage level, considering nominal levels from 75\% to 99\%. The dashed diagonal line represents perfect calibration.}
\label{fig:calibration-wind}
\end{figure}


Although the conformalization procedures discussed in the paper aim at guaranteeing marginal coverage in the sense of (\ref{eq:mvp}), we investigated in Table \ref{tab:conditional} the conditional coverages in the wind dataset with the out-of-bag conformal prediction using projected random forests defined in Algorithm \ref{algo:prf}. We binned each feature according to their four quartiles, checking the empirical coverages on the test set separately for each bin. The results show that conditional coverage stays close to the nominal level for this dataset. See \cite{dewolf2023} for a detailed discussion of this and related issues.

Since the projected random forest model is conformalized through the out-of-bag procedure, we also constructed a calibration plot for the wind dataset by comparing several nominal coverage levels with the corresponding empirical coverages observed on the test set. The results in Figure \ref{fig:calibration-wind} show good agreement between nominal levels and empirical coverage for this dataset.


\begin{table}[t]
\centering
\caption{Conditional coverages for the wind direction dataset.}
\medskip
\label{tab:conditional}
\small
\begin{tabular}{llr}
\toprule
Variable & Bin & Empirical Coverage \\
\midrule
Air temperature & Q1 & 0.905 \\
                & Q2 & 0.874 \\
                & Q3 & 0.905 \\
                & Q4 & 0.897 \\
\midrule
Dew point temperature & Q1 & 0.907 \\
                      & Q2 & 0.899 \\
                      & Q3 & 0.888 \\
                      & Q4 & 0.885 \\
\midrule
Humidity & Q1 & 0.907 \\
         & Q2 & 0.900 \\
         & Q3 & 0.879 \\
         & Q4 & 0.894 \\
\midrule
Pressure & Q1 & 0.896 \\
         & Q2 & 0.878 \\
         & Q3 & 0.892 \\
         & Q4 & 0.914 \\
\midrule
Maximum gust & Q1 & 0.881 \\
             & Q2 & 0.871 \\
             & Q3 & 0.914 \\
             & Q4 & 0.917 \\
\midrule
Wind speed & Q1 & 0.876 \\
           & Q2 & 0.880 \\
           & Q3 & 0.911 \\
           & Q4 & 0.918 \\
\bottomrule
\end{tabular}
\end{table}

\section{Concluding remarks}\label{sec:concl}


One possibility for future work is to extend the methods developed in this paper to directional data problems, in which the response variables are observed as three-dimensional vectors. It would also be interesting to investigate circular extensions of calibration-set-free conformalization procedures connected to model stacking, such as the approach proposed in \cite{marquesf2025b}.


Open source software coded in \texttt{R} \cite{R} and data for all the examples in the paper are available at:
\begin{center}
  \texttt{\small https://github.com/paulocmarquesf/circular}
\end{center}
We have two folders in this repository, named \texttt{synthetic} and \texttt{wind}, corresponding to the analyses of the respective datasets discussed in Section \ref{sec:experiments}. Inside each folder, a suffix on a script name identifies the method used to produce the prediction intervals. Suffixes \texttt{projected\_normal} and \texttt{circular\_forest} refer to implementations of split conformal prediction, using the projected normal linear model and the circular forest, respectively. The suffix \texttt{projected\_random\_forests} refers to applications of Algorithm \ref{algo:prf}.

\section*{Acknowledgements}

Paulo C. Marques F. receives support from FAPESP (Fundação de Amparo à Pesquisa do Estado de São Paulo) through project 2023/02538-0.

\appendix

\setcounter{equation}{0}
\renewcommand\theequation{A.\arabic{equation}}

\section*{Appendix A. Proof of Lemma \ref{lmm:order}}

Let $U_{(1)}\leq U_{(2)}\leq\dots\leq U_{(n)}\leq U_{(n+1)}$ denote the order statistics of the full set $\{U_1,U_2,\dots,U_n,U_{n+1}\}$. Considering that there may be ties, it follows from the definition of $U_{(k)}$ that at least $k$ of the $U_i$'s are less than or equal to $U_{(k)}$, for $k=1,\dots,n+1$. Hence, ${\sum_{i=1}^{n+1} I_{\{U_i\leq U_{(k)}\}}\geq k}$, almost surely. Taking expectations and observing that, by exchangeability, the probability ${P(U_i\leq U_{(k)})}$ is the same, for ${i=1,\dots,n+1}$, we have that ${P(U_{n+1}\leq U_{(k)})\geq k/(n+1)}$. Furthermore, for $k=1,\dots,n$, notice that $U_{n+1}>V_{(k)}$ if and only if $U_{n+1}>U_{(k)}$, because $U_{n+1}$ cannot be strictly larger than itself. Hence,
\[ {P(U_{n+1}\leq V_{(k)})=P(U_{n+1}\leq U_{(k)})}, \]
for $k=1,\dots,n$, and the result follows.

\section*{Appendix B. Out-of-bag conformal prediction}\label{appendix}

The experiments in Section \ref{sec:experiments} show that the out-of-bag conformal prediction sets exhibit empirical coverage close to the specified nominal level. Since this calibration-sample-free conformalization procedure, introduced in \cite{johansson2014}, does not share the same formal coverage property (\ref{eq:mvp}) of split conformal prediction, this empirical performance motivates the search for an adequate theoretical basis. To simplify the discussion and to emphasize the generality of the argument, we frame our exposition in terms of standard regression problems with linear response, using the simplest possible conformity score. The argument extends directly to more general settings.

The loss of the coverage property (\ref{eq:mvp}) when we move from split conformal prediction to the out-of-bag procedure is due to the fact that we no longer have the same predictive model being used to compute each of the underlying conformity scores. In fact, the use of out-of-bag predictions to compute the training sample conformity scores implies that each score is computed using a different predictive model: the subforest for which the training sample unit stayed out-of-bag during the bootstrap process. This modification breaks the original distributional symmetry, making the conformity scores no longer exchangeable. Our strategy is to symmetrize the out-of-bag conformal prediction procedure by making two practically prohibitive changes. First, we consider an idealized random forest built from an exhaustive bootstrap process that includes all possible bootstrap samples. Second, we add the future sample unit to the training sample. If we are able to move from this idealized scenario to an actual random forest, controlling the distribution of some specific differences, then we can show that the out-of-bag conformal prediction sets satisfy a coverage property with a lower bound close to the one established for split conformal prediction in Section \ref{sec:score}. 

Suppose that the pairs $(X_1,Y_1),\dots,(X_n,Y_n),(X_{n+1},Y_{n+1})$ are exchangeable, with $X_i\in\mathbb{R}^d$ and $Y_i\in\mathbb{R}$. Consider an exhaustive bootstrap of this extended training sample, which includes the pair $(X_{n+1},Y_{n+1})$, producing all possible $\tilde{B}=(n+1)^{n+1}$ samples of size $n+1$ with replacement. Let $\tilde{\mathcal{O}}_i\subset\{1,\dots,\tilde{B}\}$ be the indices of the bootstrap samples for which the $i$-th sample unit was not included (stayed ``out-of-bag''), for $i=1,\dots,n+1$. Note that $|\tilde{\mathcal{O}}_i|=n^{n+1}$ and $|\tilde{\mathcal{O}}_i|/\tilde{B}\to e^{-1}=0.3678\dots$. A regression tree is built from each bootstrap sample. Since Breiman \cite{breiman2001} introduced a random selection of the available explanatory variables as candidates to decide each tree branch split, we set the random seed to an integer obtained by applying any symmetric hashing function to the bootstrap sample before training each regression tree, in order to maintain the symmetry of the whole procedure. This process results in the idealized random forest $\widetilde{\texttt{RF}}=\{\tilde{\mu}^{(j)}\}_{j=1}^{\tilde{B}}$.

Defining the idealized out-of-bag conformity scores
\begin{equation*}
  \tilde{R}_i = \Bigg|Y_i -  \frac{1}{|\tilde{\mathcal{O}}_i}| \sum_{j\in\tilde{\mathcal{O}}_i} \tilde{\mu}^{(j)}(X_i)\Bigg|
\end{equation*}
for $i=1,\dots,n+1$, if we specify a nominal miscoverage level $0<\alpha<1$, such that ${\lceil(1-\alpha)(n+1)\rceil\leq n}$, and define $\tilde{r}=\tilde{R}_{(\lceil(1-\alpha)(n+1)\rceil)}$, the data exchangeability assumption and the completely symmetric process used to define the idealized random forest $\widetilde{\texttt{RF}}$ yield the following result, by applying the same combinatorial reasoning used to prove property (\ref{eq:mvp}) in Section \ref{sec:score}.

\begin{lemma}\label{lmm:exch}
The random vector $(\tilde{R}_1,\dots,\tilde{R}_n,\tilde{R}_{n+1})$ is exchangeable and $P(\tilde{R}_{n+1}\leq\tilde{r})\geq 1 - \alpha$.
\end{lemma}

We now move to an actual random forest $\texttt{RF}=\{\hat{\mu}^{(j)}\}_{j=1}^{B}$, built from a usual bootstrap process involving a manageable number of $B$ samples of size $n$ obtained with replacement from the training sample $(X_1,Y_1),\dots,(X_n,Y_n)$, which no longer includes the future pair $(X_{n+1},Y_{n+1})$. Let $\mathcal{O}_i\subset\{1,\dots,B\}$ be the indices of the bootstrap samples for which the $i$-th sample unit was not included, for $i=1,\dots,n$. Define the out-of-bag training sample conformity scores
\begin{equation*}
  R_i = \Bigg|Y_i -  \frac{1}{|\mathcal{O}_i|} \sum_{j\in\mathcal{O}_i} \hat{\mu}^{(j)}(X_i)\Bigg|,
\end{equation*}
for $i=1,\dots,n$, and the conformity score for the future sample unit
\begin{equation*}
  R_{n+1} = \Bigg|Y_{n+1} -  \frac{1}{B} \sum_{j=1}^B \hat{\mu}^{(j)}(X_{n+1})\Bigg|.
\end{equation*}

Defining $\hat{r}=R_{(\lceil(1-\alpha)(n+1)\rceil)}$, the key idea is that if $n$ and $B$ are large enough, the stability of the random forest algorithm would allow us to move from the idealized random forest $\widetilde{\texttt{RF}}$ to the actual random forest $\texttt{RF}$, controlling the probability that $R_{n+1}$ and $\tilde{R}_{n+1}$, as well as $\hat{r}$ and $\tilde{r}$, differ too much. The following result formalizes this idea.

\begin{theorem}\label{thm:approx}
If for every $\epsilon>0$ there exists a $\delta=\delta(\epsilon)>0$, such that
\[
  P\left(\max\left\{|R_{n+1}-\tilde{R}_{n+1}|,|\hat{r}-\tilde{r}|\right\}<\epsilon/2\right)\geq 1-\delta, 
\]
then
\[
  P(R_{n+1}\leq \hat{r})\geq 1 - \alpha - \delta - h(\epsilon),
\]
in which $h(\epsilon) = P(\tilde{r}-\epsilon<\tilde{R}_{n+1}\leq\tilde{r})$. Furthermore, if the joint distribution of the idealized conformity scores is absolutely continuous, then $\lim_{\epsilon\downarrow 0} h(\epsilon)=0$.
\end{theorem}

\begin{proof}
Define $B_\epsilon=\left\{\max\left\{|R_{n+1}-\tilde{R}_{n+1}|,|\hat{r}-\tilde{r}|\right\}<\epsilon/2\right\}$. By assumption, $P(B_\epsilon^c)\leq\delta$. On the event $B_\epsilon$, the actual future score and the actual threshold are both within $\epsilon/2$ of their idealized counterparts. Therefore, if also $\tilde{R}_{n+1}\leq\tilde r-\epsilon$, then
\[
  R_{n+1}<\tilde{R}_{n+1}+\epsilon/2\leq\tilde r-\epsilon/2
  \qquad\text{and}\qquad
  \hat r>\tilde r-\epsilon/2,
\]
so that $R_{n+1}<\hat r$. Hence
\[
  \{\tilde{R}_{n+1}\leq\tilde r-\epsilon\}\cap B_\epsilon
  \subseteq
  \{R_{n+1}\leq\hat r\}.
\]
Taking probabilities and using $P(A\cap B)\geq P(A)-P(B^c)$, we obtain
\[
  P(R_{n+1}\leq\hat r)
  \geq
  P(\tilde{R}_{n+1}\leq\tilde r-\epsilon)-P(B_\epsilon^c)
  \geq
  P(\tilde{R}_{n+1}\leq\tilde r-\epsilon)-\delta.
\]
Now,
\[
  P(\tilde{R}_{n+1}\leq\tilde r-\epsilon)
  =
  P(\tilde{R}_{n+1}\leq\tilde r)
  -
  P(\tilde r-\epsilon<\tilde{R}_{n+1}\leq\tilde r).
\]
By Lemma \ref{lmm:exch}, $P(\tilde{R}_{n+1}\leq\tilde r)\geq 1-\alpha$. Therefore,
\[
  P(R_{n+1}\leq\hat r)
  \geq
  1-\alpha-\delta-P(\tilde r-\epsilon<\tilde{R}_{n+1}\leq\tilde r)
  =
  1-\alpha-\delta-h(\epsilon),
\]
which proves the first claim. For the second claim, the events $\{\tilde r-\epsilon<\tilde{R}_{n+1}\leq\tilde r\}$ decrease, as $\epsilon\downarrow 0$, to $\{\tilde{R}_{n+1}=\tilde r\}$. If the joint distribution of the idealized conformity scores is absolutely continuous, then ties occur with probability zero, so $P(\tilde{R}_{n+1}=\tilde r)=0$. By continuity of the probability measure for decreasing events, $\lim_{\epsilon\downarrow 0}h(\epsilon)=0$.
\end{proof}

Recalling from the discussion in Section \ref{sec:cp} that the lower bound on the marginal coverage property (\ref{eq:mvp}) follows directly from the lower bound on the probability ${P(R_{n+1}\leq\hat{r})}$, we come to an understanding of why the empirical coverage of out-of-bag conformal prediction sets stays close to the specified nominal level in our experiments.

\bibliographystyle{elsarticle-num} 
\bibliography{bibliography}

\end{document}